\pdfoutput=1
\documentclass[letterpaper, 10 pt, conference]{ieeeconf}
\IEEEoverridecommandlockouts
\usepackage{amsmath,amssymb}

\usepackage{amsthm}
\usepackage{graphicx}
\usepackage{booktabs}
\usepackage{multirow}
\usepackage{url}
\usepackage[hidelinks]{hyperref}
\hypersetup{pdftitle={Median Temporal Ensembling: Training-Free Robust Aggregation for Action-Chunked Visuomotor Policies}, pdfauthor={Yuhang Jiang}}
\usepackage{xcolor}

\newtheorem{definition}{Definition}
\newtheorem{theorem}{Theorem}
\newtheorem{corollary}{Corollary}

\title{\LARGE \bf
Median Temporal Ensembling: Training-Free Robust Aggregation
for Action-Chunked Visuomotor Policies}

\author{Yuhang Jiang%
\thanks{Yuhang Jiang is with the University of Trento, Italy. {\tt\small jyhtjtj@gmail.com}. Project page: \url{https://avalon-s.github.io/MedianTE/}}%
}

\begin{document}
\maketitle
\thispagestyle{empty}
\pagestyle{empty}

\begin{abstract}
Action-chunked visuomotor policies predict overlapping trajectories, so every executed
action is covered by several predictions. %
Temporal ensembling smooths execution by combining
these predictions with an exponentially weighted mean. One corrupted prediction can move
the aggregate without bound: its breakdown point is $0$. We use adversarial corruption to
stress this deployed aggregator and to compare two kinds of guarantee. A \emph{metric}
guarantee bounds the response to a perturbation of a given size. A \emph{combinatorial}
guarantee instead bounds the damage when at most $q$ of the $M$ candidates covering a
timestep are corrupted, whatever their size. %
Encoder adversarial fine-tuning recovers $44\%$ of the loss
under the published patch attack, but only $7.3\%$ after the attacker's step size is
increased. By contrast, the coordinate-wise median of the same candidate set keeps its
recovered fraction flat as attack optimisation increases. %
Median temporal ensembling costs one line and requires no
retraining. Across 25 (configuration, corruption-level) combinations it is never worse than
the mean and is significantly better in 15. It also transfers to a second policy class, and
it recovers performance under a failure with no attacker in the loop at all: camera
frames that arrive blank.
Its effect on clean data is configuration-dependent, from $-0.04$ to $+0.07$. We also give
the boundary: corruption that shifts every covering prediction by the same amount is
invisible to this whole family of statistics, and no equivariant aggregator can remove it.
\end{abstract}

\section{Introduction}
\label{sec:intro}

Action-chunked visuomotor policies~\cite{act,dp} predict a short trajectory from each
observation but execute only its first few steps. Each control timestep is therefore
covered by overlapping chunks. A robot arm
reaching for a can re-plans repeatedly on the way, and at any moment several of those
plans still have something to say about where the gripper goes next.
Temporal ensembling, introduced with ACT~\cite{act}
and reused widely since, combines their predictions into a single command by taking an
exponentially weighted \emph{mean}.

A single corrupted prediction out of $M$ can drive that mean arbitrarily far, pushing the
arm against its limits rather than producing a bounded error. In the language of robust
statistics, the aggregator has \textbf{breakdown point $0$}. The vulnerability belongs to
the estimator, not to any particular adversary.
Recent work learns to select among the cached
predictions~\cite{tas}; the weighted mean itself has not, to our knowledge, been
analysed as an estimator.

To probe that vulnerability we need to control how many of the covering predictions are
corrupted, and an
adversarial stress test gives exactly that control. We organise the paper around \emph{what
kind of guarantee} a defence provides. A \textbf{metric} guarantee bounds the response to a
perturbation \emph{of a given size}; a \textbf{combinatorial} one bounds the damage when
\emph{at most $q$ of $M$} inputs are corrupted.

We show that this distinction predicts which defences
survive a stronger attack and which lose their effect, on the same policies and tasks
(Fig.~\ref{fig:main}). The \emph{temporal structure} of the corruption decides whether a combinatorial
guarantee exists at all: corruption that shifts every covering prediction by the same
amount produces no disagreement between candidates, leaving nothing for a rank statistic,
or for any consistency monitor, to work with.

\begin{figure}[tb]
\centering
\includegraphics[width=\columnwidth]{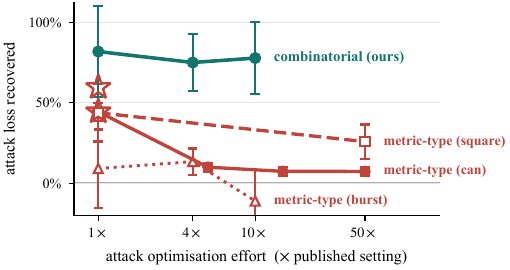}
\caption{Metric-type robustness degrades as attack optimisation increases;
combinatorial robustness persists. The vertical axis is the fraction of attack-induced
loss the defence recovers and the perturbation bound is fixed throughout. The horizontal
axis is optimisation effort relative to each attack's own baseline, measured in PGD
iterations for the two burst curves and in optimiser step size for the patch curves. Stars mark released-patch results;
error bars are $\pm1$ standard error, paired for the burst curves and binomial for the patch curves.}
\label{fig:main}
\end{figure}

\subsection{Contributions}

\textbf{Action chunking already carries the redundancy that robust aggregation needs}
(Section~\ref{sec:method}). Overlapping chunks give several estimates of the same
executed action, and temporal ensembling spends that redundancy on smoothing. Taking the coordinate-wise median of the same
candidate set instead costs one line. Over 25 (configuration, corruption-level)
combinations spanning two environments, three tasks, two backbones and four execution
strides, the median is never worse than the mean, and significantly better in 15 after
Holm correction. Our largest effect, \texttt{can} at $h{=}2$, $q{=}3$ over $200$ paired
episodes, takes success from $0.335$ to $0.675$. The same one-line change also recovers
performance under a failure with no attacker behind it: blank camera frames take
\texttt{can} from $1.00$ to $0.510$, and the median returns it to $0.690$. It carries to a second policy class as
well, where an ACT policy at $q{=}2$ goes from $0.010$ to $0.595$ and the mean wins no
episode.

\textbf{A direct measurement of why it works} (Section~\ref{sec:method}). Splitting the
candidates at every timestep by whether their query was corrupted, the median leaves the
range of the clean candidates exactly zero times in the $53\,704$ timesteps at which
fewer than half the covering candidates were corrupted,
while the mean leaves it in $81$--$93\%$ of them.

\textbf{A metric-type defence, and the mechanism of its failure}
(Section~\ref{sec:metric}). Encoder adversarial fine-tuning, ported from~\cite{edpa} to a
diffusion policy, recovers 44\% of the loss caused by the published patch attack on
\texttt{can}, and 7.3\% once the attacker's step size is raised. A gradient probe shows this is
\emph{not} gradient masking: the encoder really did become less sensitive, which is
exactly why it fails against a scale-invariant optimiser.

\textbf{Where neither route helps} (Section~\ref{sec:neither}). We
formalise the consistency statistics used by recent runtime monitors as those invariant
under a \emph{common translation} of the candidate set, and observe that common-mode
corruption is exactly such a translation, invisible to every member of the family
regardless of magnitude. %
Empirically the predicted blind spot appears at the operating points these
methods prescribe: the monitors rank attacked episodes
well and still stay silent through a total task failure.

The benefit is configuration-dependent, and Section~\ref{sec:method} names the
configurations in which it vanishes. Against persistent common-mode corruption,
Theorem~\ref{thm:cm} shows that no aggregator in the family can help.

\section{Related Work}
\label{sec:related}

\textbf{Action chunking, temporal ensembling and the monitors built on them.}
ACT~\cite{act}, Diffusion Policy~\cite{dp} and Octo~\cite{octo} emit a horizon of future
actions per query and re-plan before it is consumed. ACT's temporal ensembling, defined in
Section~\ref{sec:method}, is the weighted mean our method replaces. %
It needs both halves of that
behaviour, overlapping predictions \emph{and} a rule that combines them, and not every
chunked policy provides both:
OpenVLA~\cite{openvla} emits one action per query, while $\pi_0$~\cite{pi0} re-plans after
$16$ or $25$ of its $50$ predicted actions and discards the unexecuted remainder, having
tried temporal ensembling and found that it hurt performance. Both are statements about clean
performance, which is also where chunking's implicit ensembling has been
analysed~\cite{chunkwhy}. TAS~\cite{tas} trains a selector over the same cached
predictions with reinforcement learning; our replacement needs no training.
Several recent systems instead score the mutual agreement of those overlapping chunks:
STAC/Sentinel~\cite{stac} with a distributional distance, ACC~\cite{vlafail} with a
chunk-to-chunk mean absolute error divided by an in-chunk range, TIDE~\cite{rewindil} with
the squared error against the temporally ensembled plan, and Falcon~\cite{falcon} with a
closely related quantity used as a \emph{reuse gate for inference acceleration} rather than
a failure detector. Section~\ref{sec:neither} shows all of these share one structural
property. We re-implement TIDE and ACC from their published equations and evaluate them
under two corruption structures with identical code and identical threshold rules.

\textbf{Attacks and defences on visuomotor policies.} Perturbing a policy's observations
degrades control directly~\cite{huang2017adversarial}, and physical adversarial
patches~\cite{brown2017patch,kurakin2016physical} carry that threat into the world:
DP-Attacker~\cite{dpattacker} attacks
Diffusion Policy and releases the trained patch artefacts we use as the published baseline,
which we find under-optimised (Section~\ref{sec:calib}). Xu et al.~\cite{edpa}, alongside
their EDPA patch attack on vision-language-action models, propose fine-tuning the
observation encoder for invariance to the perturbation; we port that recipe to a diffusion
policy and characterise where most of its benefit disappears
(Section~\ref{sec:metric}). A separate family guarantees robustness over \emph{patch position} by masking
and voting~\cite{patchcleanser,patchcure,certmask}: a combinatorial guarantee over image
regions instead of over time. We do not evaluate them: masking and voting multiplies
forward passes, which compounds with a diffusion policy's multi-step sampler. As far as we
know the cheaper variants have not been demonstrated on such a policy. We do measure the
\emph{premise} such pipelines rest on (Section~\ref{sec:neither}).

\textbf{Robust aggregation.} The coordinate-wise median and trimmed estimators are
standard robust location estimators~\cite{huber1964robust}; tolerating a bounded
number of arbitrarily corrupted inputs is the Byzantine
setting~\cite{lamport1982byzantine}. Yin et al.~\cite{yin2018byzantine} analyse the
coordinate-wise median itself as a Byzantine-robust aggregation rule for distributed
learning, alongside related rules such as Krum~\cite{blanchard2017byzantine},
Bulyan~\cite{elmhamdi2018hidden} and geometric-median aggregation~\cite{pillutla2022robust}. In all of them the $M$ estimates
come from workers instead of one policy's overlapping predictions.
Two things kept the connection out of view. Temporal ensembling entered as a smoothing
device, its exponential weights motivated by recency and not by estimator theory; and on
clean data the two aggregators are identical in five of our eight configurations and differ
by at most $0.07$ in the rest, so clean-data evaluation gives little signal that the
vulnerability exists.

\section{Overlapping Predictions}
\label{sec:theory}

\noindent\textbf{Setup.}
An action-chunked policy $\pi$ is queried every $h$ control steps ($T_a$ in~\cite{dp}). The $i$-th query occurs
at control time $t_i = i\,h$, takes the most recent $T_o$ observations, and emits a chunk
$A^{(i)} = \pi(o_i) \in \mathbb{R}^{T_p \times d}$. The first $T_o-1$ entries align with
observation steps already in the past, so the executable part of the chunk starts at index
$T_o-1$. Under receding-horizon execution only $h$ entries are executed per query, so any
absolute time $\tau$ is predicted by several queries. Write
\[
  C(\tau) = \{\, i : 0 \le \tau - t_i \le T_p - T_o \,\},
\]
\[
  a^{(i)}(\tau) := A^{(i)}_{\tau - t_i + T_o - 1}
\]
for the \emph{covering set} of $\tau$ and its \emph{candidates}.

\textbf{Coverage number.} $|C(\tau)|$ is \emph{not} constant: it alternates with the phase
of $\tau$ relative to the query grid,
between $M_{\min} = 1 + \lfloor (T_p - T_o - (h-1))/h \rfloor$ and
$M_{\max} = \lfloor (T_p - T_o)/h \rfloor + 1$.
For $T_p{=}16$, $T_o{=}2$, $h{=}2$ this gives $M \in \{7,8\}$, measured at 48.8\% and
48.2\% of timesteps respectively, with $M\le 6$ only at episode boundaries (3\%).
Concretely, queries then fire at even control times, so $\tau{=}100$ is covered by the
eight queries issued at $t = 86, 88, \dots, 100$, while $\tau{=}101$ is covered by seven:
the query at $t{=}86$ has run out of executable entries by then. What decides which of
the two a timestep gets is its phase against the query grid. We write
$q(\tau)$ for the number of \emph{corrupted} candidates covering $\tau$
(Section~\ref{sec:threat}).

\section{Threat Model and Attack Calibration}
\label{sec:threat}

\noindent %
We choose the threat model to expose the breakdown point of a deployed aggregator. It
covers two structural regimes. Under \textbf{sparse} corruption at most some of the $M$
covering candidates are corrupted; under \textbf{common-mode} corruption every covering
candidate is shifted by the same $b(\tau)$ (Section~\ref{sec:neither}).

\subsection{Sparse Query Corruption}
\label{sec:sparse}

The attacker, or fault source, may arbitrarily perturb the observation of selected queries
within an $\ell_\infty$ ball of radius $\epsilon$~\cite{madry2018towards}; the remaining
queries are clean.
Following standard adaptive-evaluation practice~\cite{tramer2020adaptive} we grant full
knowledge of the aggregator,
the stride $h$ and the schedule, and let the attacker choose \emph{which} queries to
corrupt. The objective it optimises, however, is the policy's own loss, so this attacker
is the same whichever aggregator is deployed; Section~\ref{sec:method} adds one that
targets the median directly. Our main schedule is a contiguous burst of $q$ consecutive queries out of every
period of $M_{\min}$, because a contiguous run maximises the corrupted fraction at the
worst-hit timesteps.

$q$ is a schedule parameter. A covering set spans seven or eight consecutive queries
while the schedule repeats every $M_{\min}{=}7$, so a timestep late in one period sees
the tail of that period's burst together with the head of the next, and the corrupted
candidate count $q(\tau)$ routinely \emph{exceeds} the nominal $q$. We measure it. The distinction matters because the guarantee of Section~\ref{sec:method} needs a
\emph{strict} inequality, $2q(\tau) < M(\tau)$: at a tie the median averages the two
central order statistics, one of which is already corrupted. Ties are not a corner case,
because $M$
alternates: at one and the same nominal $q$ the timesteps with the larger $M$ can tie while
those with the smaller $M$ cannot.

\subsection{Sensor Faults and Intermittent Access}

The breakdown point is
attacker-independent: it asks how much of the input can be arbitrarily corrupted
before the estimate becomes unbounded. Intermittent occlusion, specular glare, motion
blur, dropped frames and transient sensor faults produce that same pattern, and
robustness to such non-adversarial corruption is studied in its own
right~\cite{hendrycks2019benchmarking}. ``At most $q$ of the covering predictions are untrustworthy'' is an
assumption about the \emph{quality of the redundancy}.
Among these faults, blank frames are by far the most damaging, of the kind a driver that
zero-fills on timeout produces: at $q{=}2$ of every seven queries they take
\texttt{can} from $1.00$ to $0.510$, as damaging as the calibrated
attack at the same density, and Section~\ref{sec:method} evaluates the aggregators there.
A single stale frame, an intermittent occluder swept to $48$ of $84$ pixels, motion blur
and a camera frozen for the whole burst each cost the undefended policy at most $0.08$
in success rate, even with every view occluded.

Corrupting every query is demanding in its own right: per-query optimisation at
inference time requires white-box gradients and per-frame compute, so access to the
observation stream may itself be intermittent.

\noindent\textbf{Scope.} The threshold $2q(\tau) < M(\tau)$ governs whether the aggregator's
output stays inside the clean range. It does not tell us where task success collapses;
that happens at $q{=}4$ on \texttt{can} and $q{=}5$ on \texttt{square}. The argument
applies unchanged when the $M$ estimates come from several cameras or agents, which we do
not test.

\subsection{Calibrating the Attacks}
\label{sec:calib}

How robust a defence looks depends on how hard the attack against it was
optimised~\cite{carlini2019evaluating}, and the gap is large here
(Section~\ref{sec:metric}). We therefore calibrate the attack before evaluating anything
against it.

\noindent\textbf{The physical-patch attack.}
All arms use the same optimisation budget (11 epochs, same seed, same batch size) and
differ only in the step size and the update rule: the raw gradient, or its sign, as
standard PGD~\cite{madry2018towards} uses.

\begin{table}[tb]
\centering
\caption{Calibration of the physical-patch attack. Cells are the undefended policy's
success rate under the resulting patch, so lower means a stronger attack; $n{=}200$
throughout. Arm A is the released artefact, arm B the configuration stated
in~\cite{dpattacker}, and arm C raises only the step size.}
\label{tab:calib}
\begin{tabular}{llcc}
\toprule
arm & update rule & \texttt{can} & \texttt{square} \\
\midrule
A  & released artefact (raw gradient) & 0.175 & 0.680 \\
B  & as specified: sign-PGD, $\alpha{=}10^{-4}$ & 0.185 & 0.700 \\
C  & sign-PGD, $\alpha{=}5\!\times\!10^{-3}$ & \textbf{0.025} & \textbf{0.475} \\
\bottomrule
\end{tabular}
\end{table}

The published attack is under-optimised in both its released configuration and under the
hyper-parameters its paper states (Table~\ref{tab:calib}). Reproducing the paper's own update rule at the
paper's own step size (arm B) yields a patch indistinguishable in strength from the released
one (\texttt{can}: $z{=}-0.26$, $p{=}0.80$; \texttt{square}: $z{=}-0.43$, $p{=}0.67$).
But raising only the step size takes \texttt{can} from $0.185$ to $0.025$
($z{=}5.22$, $p{=}1.8\!\times\!10^{-7}$) and \texttt{square} from $0.700$ to $0.475$
($z{=}4.57$, $p{=}4.9\!\times\!10^{-6}$).

The attack saturates early. On \texttt{can}, $\alpha{=}5\!\times\!10^{-4}$ already brings
the undefended success rate to $0.030$, and a further tenfold increase moves it only to
$0.025$ (Table~\ref{tab:metric}). The published setting therefore sits a factor of five in step
size below the point where this attack stops improving.
A defence evaluated against the published configuration faces an attack well
short of its own optimum.

\noindent\textbf{The burst attack.}
For the sparse threat of Section~\ref{sec:sparse} the attacker runs projected gradient descent (PGD) on the observation
at the attacked queries ($\epsilon{=}0.03$, step $0.001875$). We calibrate the iteration
budget the same way, by pushing it until it stops helping: undefended success on
\texttt{can} $h{=}2$, $q{=}2$ is $0.780$ at 50 steps, $0.560$ at 200 and $0.640$ at 500.
The attack saturates by 200 steps, which is the budget we use throughout, and we keep the
$10\times$ sweep as a robustness check.

Throughout, we report effectiveness as the \emph{fraction of attack-induced loss
recovered}, since a stronger attack enlarges the recoverable headroom and would otherwise
inflate raw success-rate differences.

\section{Median Temporal Ensembling}
\label{sec:method}

We replace temporal ensembling's exponentially weighted mean, $w_i \propto
\exp(-m\cdot\mathrm{age}_i)$, with the \textbf{coordinate-wise median} of the same $M$
candidates covering a control timestep, and change nothing else. Here
$\mathrm{age}_i$ is how many queries ago candidate $i$ was issued ($0$ for the most
recent of the $M$), $m$ sets how fast older predictions are discounted, and we keep the
deployed $m{=}0.01$. The substitution is a single line: no
retraining, no additional forward passes, and under $0.3\%$ of the cost of one
decision.

\subsection{The Bracket Guarantee}

One corrupted candidate suffices to move the mean without
bound: its finite-sample breakdown fraction is $1/M$ and its asymptotic breakdown point
is $0$. The coordinate-wise median
stays within the range of the clean candidates whenever fewer than half of them are
corrupted, the property that makes it a Byzantine-robust
aggregator~\cite{yin2018byzantine}. The guarantee is \emph{combinatorial}: it depends on how
many candidates are corrupted, not on how large the perturbation is.

That guarantee is already worst-case over the corrupted candidates. When
$2q(\tau)<M(\tau)$ the median's rank lies between $q{+}1$ and $M{-}q$. At most $q$
candidates are corrupted, so the $(q{+}1)$-th order statistic is at least the smallest clean
candidate and the $(M{-}q)$-th is at most the largest; the median is therefore confined to
the range of the clean candidates, whatever values the corrupted ones take. We call that
range the \emph{bracket}. No choice of
corrupted candidates, adaptive or otherwise, can move the output outside that range. An
attacker who knows the aggregator therefore has two levers: corrupt more candidates, the
$q$-axis we sweep, or exploit the freedom left \emph{inside} the bracket.

Our burst attacker optimises each corrupted query independently. The bracket argument holds
for arbitrary corrupted values, so the statistic-level claim is unaffected; the task-level
question is whether the freedom remaining \emph{inside} the clean bracket is enough to make
the task fail. An oracle that replaces every corrupted candidate by the value that moves the
median farthest does not make it fail: success stays at $0.975$ and $0.955$ on \texttt{can} at
$q{=}2,3$ and at $0.835$ and $0.840$ on \texttt{square} at $q{=}3,4$ ($n{=}200$ each).
It is the worst case over everything the bracket leaves free, excluding attacks that also
act where the bracket does not hold. How the attacker spends that
freedom matters: moving the median the same way at consecutive timesteps costs more than
moving it greedily in all four cells of Table~\ref{tab:oracle}, and significantly more in
two of them after Holm correction. No cell fails the task, and our headline attack does
not exploit the direction at all.
\begin{table}[t]
\centering
\caption{Oracle attacks confined to the bracket ($h{=}2$, $n{=}200$ paired episodes).
``inst.'' moves the median as far as the bracket allows at each timestep independently;
``drift'' spends the same freedom consistently over time. ``paired'' counts episodes won
only by inst.\ versus only by drift. A star marks the cells surviving Holm correction over
the four.}
\label{tab:oracle}
\footnotesize
\begin{tabular}{@{}llcccl@{}}
\toprule
task & $q$ & inst. & drift & paired & $p$ \\
\midrule
\texttt{can} & 2 & $0.975$ & $0.960$ & $6{:}3$ & 0.51 \\
\texttt{can} & 3 & $0.955$ & $0.880$ & $21{:}6$ & $5.9\!\times\!10^{-3}$* \\
\texttt{square} & 3 & $0.835$ & $0.650$ & $54{:}17$ & $1.3\!\times\!10^{-5}$* \\
\texttt{square} & 4 & $0.840$ & $0.760$ & $34{:}18$ & 0.036 \\
\bottomrule
\end{tabular}
\end{table}

\subsection{Weighting versus Statistic}

Because replacing the deployed aggregator changes two things at once, we separate them on the one cell for which we ran all six aggregators, holding the
candidate set and the episode seeds fixed (Table~\ref{tab:agg}). Uniform weights and exponential weights give the \emph{same}
success rate ($0.600$ both, paired $3{:}3$); every rank statistic gains at least $0.28$.
The gain therefore comes from the choice of statistic.

No two rank statistics separate in our data. The block medoid is the candidate closest to
the others; the trimmed medoid drops the $q$ most outlying candidates first. Comparing the
median against both at $q{\in}\{2,3\}$ on $n{=}200$ paired episodes, none of the six
pairwise comparisons survives Holm correction, and pooling the two grids leaves the
strongest at
$p{=}0.044$ against a $0.017$ threshold; at that sample size the medoids' point
estimates are higher by $0.035$ to $0.050$.

\begin{table}[tb]
\centering\footnotesize
\caption{Aggregators on the same candidate set and episode seeds (\texttt{can}, $h{=}2$, $q{=}2$, $n{=}50$). ``paired'' counts episodes won only by the deployed weighted mean versus only by the row.}
\label{tab:agg}
\begin{tabular}{lccc}
\toprule
aggregator & success & paired & $p$ \\
\midrule
latest prediction & 0.480 & $10{:}4$ & 0.18 \\
unweighted mean & 0.600 & $3{:}3$ & 1.00 \\
\textit{exp.-weighted mean (deployed)} & \textit{0.600} & --- & --- \\
coordinate-wise median (ours) & 0.880 & $2{:}16$ & \textbf{1.3$\times$\textbf{10}$^{-3}$} \\
block medoid & 0.940 & $2{:}19$ & \textbf{2.2$\times$\textbf{10}$^{-4}$} \\
trimmed medoid & 0.960 & $1{:}19$ & \textbf{4.0$\times$\textbf{10}$^{-5}$} \\
\bottomrule
\end{tabular}
\end{table}

\subsection{Measuring the Breakdown Point}

To see this directly, at every execution timestep we split the candidates by whether their
query was corrupted,
take the coordinate-wise range of the
\emph{clean} candidates, and record whether each aggregator's output falls outside it.
Among timesteps that contain at least one corrupted candidate
(Table~\ref{tab:mech}), %
the guaranteed regime covers $53\,704$ timesteps across two tasks
and four corruption levels, and the median leaves the clean range \textbf{exactly zero
times} (maximum normalised excess $0.0000$), while the mean leaves it in $81$--$93\%$ of
them. The threshold is sharp, and it is a property of the aggregator's output.

The executed action is one no query proposed: taking coordinates independently assembles a
vector no candidate offered. This is worth checking: an action no query
endorsed could be one the policy would never produce. Over the $82\,432$ corrupted timesteps of \texttt{can} at
$h{=}2$, $q{=}2$ the output coincides with a candidate in $0.03\%$ of them. With $M$ odd
each coordinate is some candidate's value but not the same one's, and with $M$ even the two
central order statistics are averaged.
The assembled action stays close to what it is assembled from: its distance to the nearest
candidate is $0.13$ of the mean pairwise distance between candidates (p99 $0.51$), and it
never leaves their coordinate-wise range.

\begin{table}[t]
\centering
\caption{How often each aggregator's output leaves the range of the clean candidates, and
by how far. ``excess'' is the median excursion beyond that range over corrupted
timesteps, in units of its width. Each entry spans \texttt{can} $q\in\{2,3,4,5\}$ and
\texttt{square} $q\in\{3,4\}$ at $h{=}2$.}
\label{tab:mech}
\footnotesize
\begin{tabular}{@{}lcccc@{}}
\toprule
 & \multicolumn{2}{c}{outside (\%)} & \multicolumn{2}{c}{excess} \\
\cmidrule(lr){2-3}\cmidrule(lr){4-5}
regime & median & mean & median & mean \\
\midrule
$2q < M$ \ (guaranteed) & \textbf{0.0} & 81--93 & \textbf{0.00} & 0.54--1.64 \\
$2q = M$ \ (tie)        & 71--84       & 92--99 & 0.24--0.87 & 1.88--2.22 \\
$2q > M$                & 87--96       & 98--99 & 2.4--15.3 & 4.4--42.4 \\
\bottomrule
\end{tabular}
\end{table}

\begin{figure*}[t]
\centering
\includegraphics[width=\textwidth]{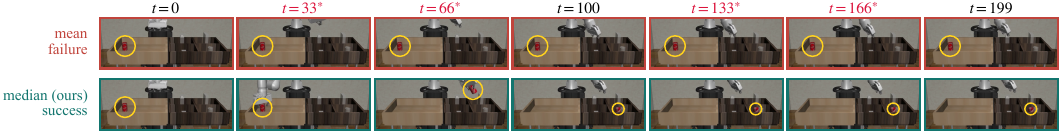}
\caption{One episode of \texttt{can} at $h{=}2$, $q{=}3$, run twice from the same seed and
under the same attack, changing only the aggregator: the weighted mean (top) fails and
the median (bottom) succeeds. A circle marks the can in each frame; columns
are equally spaced in time and a star marks a corrupted query window. Frames are a third-person render for
visualisation only; the policy sees two $84\times84$ views.}
\label{fig:qual}
\end{figure*}

The breakdown point explains \emph{why} a rank-based aggregator resists corruption that an
averaging one cannot, and correctly predicts that the advantage fades as corruption
spreads. The threshold is sharp for the statistic but not for the task: at
\texttt{can} $h{=}1$, $q{=}8$ every covering set is corrupted past the bracket
condition, and the median still gains $0.240$. Table~\ref{tab:mech} shows why: once
the guarantee lapses the median does leave the clean range, but its excursion stays a
fraction of the mean's. The breakdown point does \textbf{not}, however, predict the magnitude of the benefit, which differed
by a factor of two across configurations matched on the fraction of unprotected timesteps,
and was not monotone in $M$.

\subsection{Results}

\noindent\textbf{Setup.} Sparse burst corruption on the schedule and at the budget
calibrated in Section~\ref{sec:calib}. In the robomimic tasks a $7$-DoF arm must lift a
cube (\texttt{lift}), move a can into a bin (\texttt{can}) and seat a square nut on a peg
(\texttt{square}); PushT asks a circular end-effector to push a T-shaped block onto a
target pose. Both aggregators see the \emph{same} episode seeds, so every test is paired.

\begin{table*}[t]
\centering\footnotesize
\caption{Median vs.\ weighted-mean temporal ensembling on diffusion policies under sparse burst corruption. ``clean'' is the same-configuration no-attack baseline, mean $\to$ median. $p$ is raw; a star marks the cells surviving Holm correction over all 25. McNemar counts read mean-only\,:\,median-only successes; PushT reports coverage, tested by Wilcoxon. $n{=}50$ unless noted.}
\label{tab:main}
\begin{tabular}{llccccccl}
\toprule
env / task & bb & $h$ & clean & $q$ & mean & median & $\Delta$ & $p$ \\
\midrule
robomimic \texttt{can} & CNN & 1 & 0.98$\to$0.98 & 5 & 0.520 & 0.820 & +0.300 & \textbf{2.7$\times$\textbf{10}$^{-4}$}* \\
 & & &  & 7 & 0.220 & 0.480 & +0.260 & \textbf{7.2$\times$\textbf{10}$^{-3}$} \\
 & & &  & 8$^{n=100}$ & 0.120 & 0.360 & +0.240 & \textbf{4.8$\times$\textbf{10}$^{-5}$}* \\
\midrule
robomimic \texttt{can} & CNN & 2 & 1.00$\to$1.00 & 2 & 0.600 & 0.900 & +0.300 & \textbf{2.7$\times$\textbf{10}$^{-4}$}* \\
 & & &  & 3$^{n=200}$ & 0.335 & 0.675 & +0.340 & \textbf{9.2$\times$\textbf{10}$^{-13}$}* \\
 & & &  & 4 & 0.220 & 0.300 & +0.080 & 0.42 \\
 & & &  & 5 & 0.080 & 0.140 & +0.060 & 0.38 \\
\midrule
robomimic \texttt{can} & CNN & 3 & 1.00$\to$0.96 & 1 & 0.800 & 1.000 & +0.200 & \textbf{2.0$\times$\textbf{10}$^{-3}$}* \\
 & & &  & 2 & 0.300 & 0.580 & +0.280 & \textbf{1.3$\times$\textbf{10}$^{-3}$}* \\
 & & &  & 3 & 0.080 & 0.260 & +0.180 & \textbf{1.2$\times$\textbf{10}$^{-2}$} \\
\midrule
robomimic \texttt{can} & CNN & 4 & 1.00$\to$1.00 & 1 & 0.320 & 0.560 & +0.240 & \textbf{4.2$\times$\textbf{10}$^{-3}$}* \\
 & & &  & 2 & 0.020 & 0.080 & +0.060 & 0.25 \\
\midrule
robomimic \texttt{square} & CNN & 2 & 0.98$\to$0.98 & 2 & 0.860 & 0.900 & +0.040 & 0.62 \\
 & & &  & 3$^{n=100}$ & 0.670 & 0.840 & +0.170 & \textbf{2.1$\times$\textbf{10}$^{-3}$}* \\
 & & &  & 4$^{n=200}$ & 0.410 & 0.600 & +0.190 & \textbf{4.4$\times$\textbf{10}$^{-5}$}* \\
 & & &  & 5 & 0.380 & 0.380 & +0.000 & 1.00 \\
\midrule
robomimic \texttt{lift} & CNN & 2 & 1.00$\to$1.00 & 3 & 1.000 & 1.000 & +0.000 & 1.00 \\
 & & &  & 4 & 1.000 & 1.000 & +0.000 & 1.00 \\
\midrule
PushT & CNN & 2 & 0.79$\to$0.86 & 1 & 0.750 & 0.844 & +0.095 & \textbf{1.5$\times$\textbf{10}$^{-3}$}* \\
 & & &  & 2 & 0.643 & 0.842 & +0.199 & \textbf{8.2$\times$\textbf{10}$^{-7}$}* \\
 & & &  & 3 & 0.609 & 0.776 & +0.167 & \textbf{7.8$\times$\textbf{10}$^{-4}$}* \\
 & & &  & 4 & 0.515 & 0.689 & +0.174 & \textbf{1.5$\times$\textbf{10}$^{-4}$}* \\
\midrule
PushT & TF & 2 & 0.66$\to$0.68 & 1$^{n=200}$ & 0.553 & 0.623 & +0.069 & \textbf{1.3$\times$\textbf{10}$^{-4}$}* \\
 & & &  & 2$^{n=200}$ & 0.372 & 0.450 & +0.078 & \textbf{9.0$\times$\textbf{10}$^{-5}$}* \\
 & & &  & 3$^{n=200}$ & 0.284 & 0.323 & +0.039 & \textbf{1.6$\times$\textbf{10}$^{-2}$} \\
\midrule
\bottomrule
\end{tabular}
\end{table*}

Table~\ref{tab:main} shows the broad pattern: across 25 cells the median is \textbf{never
worse} than the mean and is significantly better in 18 at $p{<}0.05$, or \textbf{15} after
Holm correction.
Because rollouts are stochastic, repeating one $n{=}50$ cell on the same seeds shifted it
by $0.08$; we therefore made the sweep broad rather than deep.

Figure~\ref{fig:qual} shows one of the $78$ episodes in which the median succeeds and
the mean fails, drawn from the largest cell. The attack and the seed are identical; only the aggregator
changes. Across all $200$ paired episodes the median wins $78$ and the mean $10$, while
both succeed in $57$ and both fail in $55$.

PushT differs from robomimic in action dimensionality, observation structure and success
metric, which makes it our test of transfer. There the effect holds on both backbones at
every level tested: $+0.095$ to $+0.199$
with the CNN, $+0.039$ to $+0.078$ with the Transformer, which has the fewest candidates in
our sweep ($T_p{=}10$, hence $M\in\{4,5\}$). That is the least redundant configuration we
test, and the effect survives there too.

We next test blank camera frames (Section~\ref{sec:threat}), which corrupt
candidates with no adversary, no gradient and no optimisation behind them. The
median raises \texttt{can} from $0.510$ to $0.690$ at $q{=}2$ (paired $12{:}48$,
$p{=}3.2\!\times\!10^{-6}$) and from $0.120$ to $0.185$ at $q{=}3$ ($10{:}23$,
$p{=}0.035$). Both improvements survive Holm correction over the two densities; at $q{=}4$
both aggregators score $0.000$. %
The recovered fraction is $0.37$ at $q{=}2$ and $0.07$ at $q{=}3$, because a blinded policy
often proposes a wrong action that still lies inside the range of the clean candidates,
where no rank statistic can reject it. Under attack the recovered fraction
reaches $0.75$--$0.82$.

On clean data the cost is configuration-dependent: zero for \texttt{can} at $h\in\{1,2,4\}$,
for \texttt{square} and for \texttt{lift}; $-0.040$ for \texttt{can} at $h{=}3$; and
$+0.069$ / $+0.016$ on PushT. None of these three is significant ($p{=}0.48$, $0.46$,
$0.81$).
Some configurations leave nothing to recover: \texttt{lift} is immune to this threat at
every budget tested, with both aggregators succeeding on every episode.

The advantage survives a stronger attack, which is what separates the median from the
metric-type defence of Section~\ref{sec:metric}. We raise the PGD budget $10\times$, from
$50$ to $500$ steps, on \texttt{can} at $h{=}2$, $q{=}2$, and it holds at every budget
tested.

By the recovered-fraction measure of Section~\ref{sec:calib} the median is flat across the
sweep ($0.818$, $0.750$, $0.778$ at 50, 200 and 500 steps), and the paired advantage is
significant at each ($1{:}10$, $p{=}0.012$; $1{:}16$, $p{=}2.7\!\times\!10^{-4}$;
$3{:}17$, $p{=}2.6\!\times\!10^{-3}$). %
The mean is not monotone in budget, so we read this as the saturation of
Section~\ref{sec:calib}.

Finally, we repeat the substitution on an ACT policy trained on \texttt{can} (clean
success $0.85$ at $h{=}2$, against $1.00$ for the diffusion policy). Under the $200$-step
attack used everywhere else, both aggregators fall to $0.000$, so that budget cannot
distinguish them. We therefore
recalibrate it for this policy the same way as in Section~\ref{sec:calib}, lowering it
until the undefended policy still functions. At $20$ steps the median raises success from $0.010$ to $0.595$ at
$q{=}2$ (paired $0{:}117$) and from $0.000$ to $0.175$ at $q{=}3$ ($0{:}35$). Both gains
survive Holm correction over the three values of $q$ tested; at $q{=}4$ both aggregators
score $0.000$, and the mean wins no episode in any of the three. At $5$ steps the attack no longer moves the median at all: it reaches $0.880$ against its
own clean $0.850$ (paired $9{:}15$, $p{=}0.31$), while the undefended policy is already
down to $0.805$. The gain therefore
holds on a second policy class, over the budget range in which the undefended policy still
functions.

\section{A Metric-Type Defence: Encoder Fine-Tuning}
\label{sec:metric}

\noindent\textbf{The other kind of answer.}
If corruption enters through the visual encoder, the natural escape is to make the encoder
insensitive to it. We port the encoder adversarial fine-tuning proposed alongside EDPA~\cite{edpa} for
vision-language-action models to a diffusion policy: freeze the diffusion head, fine-tune
the observation encoder for 3000 steps at $10^{-5}$ on batches carrying the published patch
at a random position, with a consistency term tying the perturbed embedding to the clean
one. This is a
\emph{metric}-type defence. It certifies nothing about how many inputs are corrupted; it
aims to make the \emph{response} to a corrupted input small. Randomized
smoothing~\cite{cohen2019certified} is the certified member of this family; the patch
defences of Section~\ref{sec:related} are the combinatorial one, over image regions.

Evaluated with the released patch~\cite{dpattacker}
at $n{=}200$ per arm, fine-tuning recovers $44.4\%$ of the attack-induced loss on
\texttt{can} and $59.6\%$ on \texttt{square} %
(both $p{<}10^{-4}$, with error bars from independent binomials). The clean
cost is small but non-zero, $0.985 \to 0.970$ and $0.940 \to 0.935$. That much holds against the attack as
published. Retraining the patch
with sign-PGD and sweeping the step size upward from the value stated
in~\cite{dpattacker} takes the recovered fraction on \texttt{can} from $0.444$ to $0.099$ at
$\alpha{=}5\!\times\!10^{-4}$, ending at $0.073$, a $6.1\times$ loss; on \texttt{square} it
falls from $0.438$ to $0.258$ (Table~\ref{tab:metric}). The collapse coincides with the
attack reaching its own plateau: the undefended success rate falls from $0.185$ to $0.030$
over the same step and barely moves thereafter. The defence survives only where the attack
is under-optimised. Both strongest-attack rows stay
significant ($p{=}3.2\!\times\!10^{-3}$, $1.6\!\times\!10^{-2}$), so the effect collapses
without being eliminated.

\begin{table}[tb]
\centering\footnotesize
\caption{Encoder adversarial fine-tuning against attacks of increasing strength. All cells $n{=}200$ per arm. ``Recovered'' is $(\text{fine-tuned}-\text{undefended})/(\text{clean}-\text{undefended})$.}
\label{tab:metric}
\begin{tabular}{llccc}
\toprule
task & attack & undef. & fine-tuned & recovered \\
\midrule
\multirow{5}{*}{\texttt{can}} & released & 0.175 & 0.535 & \textbf{0.444} \\
 & sign-PGD $1\!\times\!10^{-4}$ & 0.185 & 0.540 & \textbf{0.444} \\
 & sign-PGD $5\!\times\!10^{-4}$ & 0.030 & 0.125 & \textbf{0.099} \\
 & sign-PGD $1.5\!\times\!10^{-3}$ & 0.030 & 0.100 & \textbf{0.073} \\
 & sign-PGD $5\!\times\!10^{-3}$ & 0.025 & 0.095 & \textbf{0.073} \\
\midrule
\multirow{3}{*}{\texttt{square}} & released & 0.680 & 0.835 & \textbf{0.596} \\
 & sign-PGD $1\!\times\!10^{-4}$ & 0.700 & 0.805 & \textbf{0.438} \\
 & sign-PGD $5\!\times\!10^{-3}$ & 0.475 & 0.595 & \textbf{0.258} \\
\bottomrule
\end{tabular}
\end{table}

This is not the obfuscated-gradient behaviour described by Athalye et
al.~\cite{athalye2018obfuscated}; the gradient is simply smaller. Probing the
encoder directly (same
observations, same patch, gradients through the original and the fine-tuned encoder), the
gradient \emph{direction} is preserved (cosine $0.637$ / $0.710$) while its magnitude falls
to $0.107$ / $0.211$ and the induced latent displacement falls in proportion ($0.149$ /
$0.284$). The fine-tuning did what it was meant to do, and that is why it fails: a
sign-based optimiser is scale-invariant, so a gradient ten times smaller still produces a
full-size step. Lowering sensitivity is not robustness, and more
adversarial training makes it worse: regenerating the patch against the current encoder at every step
recovers $15.6\%$ where the fixed version recovers $44.4\%$ ($n{=}200$,
$p{=}3.3\!\times\!10^{-6}$), and $3.6\%$ against the strong patch, no longer significant
($p{=}0.083$). %
The defence is also specific to the corruption it was trained against. Put on the same
burst axis as the median, against the same deployed aggregator and the same seeds, the
fine-tuned encoder recovers $0.091$, $0.134$ and $-0.111$ of the attack-induced loss at
$50$, $200$ and $500$ PGD steps; no cell is significant (smallest $p{=}0.14$) and the
sign changes with the budget. The median recovers $0.75$--$0.82$ over the same three
cells, significantly at each (Section~\ref{sec:method}). Making the encoder insensitive to
a fixed patch does not carry over to perturbations optimised per query.
What this family bounds is the \emph{response} to a perturbation, which a scale-invariant
optimiser rescales at no cost.

\section{The Common-Mode Blind Spot}
\label{sec:neither}

Both answers so far rest on the corrupted predictions disagreeing with the clean ones.
This section asks what is left when they do not.

\subsection{A Family of Consistency Statistics}

\begin{definition}[$\Delta$-family]
A statistic $S$ acting on the candidate set $A(\tau) = \{a^{(i)}(\tau)\}_{i \in C(\tau)}$
belongs to the \emph{$\Delta$-family} if it is invariant under a common translation of all
candidates: $S(\{a^{(i)}(\tau) + b\}) = S(\{a^{(i)}(\tau)\})$ for every $b \in \mathbb{R}^d$.
\end{definition}
Pairwise differences are the form most monitors take, but the distributional statistics
of~\cite{stac} need the invariance formulation.

In practice the published monitors all fall in this family. TIDE~\cite{rewindil} compares a fresh chunk with
the temporally ensembled plan, which is covered by Corollary~\ref{cor:agg} below.
STAC~\cite{stac} uses a distributional distance that is translation invariant. Falcon's
reuse gate~\cite{falcon} is a norm of a difference. ACC~\cite{vlafail} divides a difference
by an in-chunk range taken across \emph{different absolute timesteps}; it therefore belongs
to the family only when the bias is approximately constant over the overlap window.

\subsection{The Common-Mode Null Space}

\begin{theorem}
\label{thm:cm}
Suppose the corruption adds the \emph{same} bias $b(\tau) \in \mathbb{R}^d$ to every
covering candidate, $\tilde a^{(i)}(\tau) = a^{(i)}(\tau) + b(\tau)$ for all
$i \in C(\tau)$. Then $\tilde S(\tau) = S(\tau)$ for every $S$ in the $\Delta$-family,
regardless of $\|b(\tau)\|$.
\end{theorem}

\begin{proof}
Such a corruption is exactly a common translation of the candidate set, and Definition~1 is
invariance under common translation.
\end{proof}

\begin{corollary}
\label{cor:agg}
If an aggregator $g$ satisfies $g(\{x_i + b\}) = g(\{x_i\}) + b$, true of
the mean, the median, the medoid and any quantile, then ``candidate minus aggregate'' is
also in the $\Delta$-family. Comparing a fresh chunk against a temporally ensembled plan
therefore inherits the blind spot.
\end{corollary}

\begin{corollary}
\label{cor:out}
The same equivariance implies that such an aggregator's \emph{output} is shifted by exactly
$b(\tau)$. Common-mode corruption is thus neither measurable by $\Delta$-family statistics
nor removable by equivariant aggregation; robust aggregation can only recover the
non-common-mode component.
\end{corollary}

The premise is an \emph{exact} common bias on every candidate; statistics comparing
candidates against an external reference fall outside Definition~1 by construction.

\subsection{What the Blind Spot Costs}

Our faithful re-implementations of TIDE~\cite{rewindil} and ACC~\cite{vlafail} do
separate attacked from clean episodes well, with AUROC $0.85$--$0.96$ on \texttt{can}. The problem appears at the prescribed operating
point. With thresholds calibrated on $500$ clean episodes according to each paper's rule,
TIDE never fires and ACC fires on $2\%$ of episodes after the physical patch has driven
the policy's success rate to $0.04$. The monitors rank
episodes well, yet stay silent through a total task failure.

\noindent\textbf{Where the boundary lies.}
To locate that boundary we first make the burst attack dense, corrupting every query in
the window. Detection does not fall away. Against the $92\%$ and $96\%$ that TIDE and ACC reach at
$q{=}2$ on \texttt{can}, the dense burst leaves them at $98\%$ / $98\%$ there and at
$30\%$ / $92\%$ on \texttt{square},
while the policy's success rate falls to $0.00$ and $0.18$ respectively. This attack is
persistent but it is not a \emph{common translation}, which is what Theorem~\ref{thm:cm}
requires. Untargeted PGD optimises each query independently, producing candidate
displacements that are strongly \emph{aligned} (median pairwise cosine $0.712$) but
unequal in magnitude, so the candidates still disagree and the $\Delta$-family still sees
them.
A physical patch approximates the equality the theorem requires. Fixed in the scene, it
looks much the same to every query that sees it, and so enters every overlapping
prediction in almost identical form. %
Corollary~\ref{cor:out} then says a rank aggregator cannot
help there, and it does not: under the released patch both aggregators score $0.19$, and
over $200$ paired episodes only three differ ($1{:}2$, $p{=}1.00$). That
is why the monitors stayed silent under the patch yet kept firing under the dense burst:
what separates the two cases is whether the corruption left the candidates disagreeing,
and the policy was driven to failure in both.

\section{Limitations}
\label{sec:limits}

\textbf{Scope.} All results are in simulation (robomimic~\cite{robomimic}: \texttt{lift},
\texttt{can}, \texttt{square}; and PushT); none is on hardware. The physical patch is
rendered into the observation rather than printed, and the sensor faults of
Section~\ref{sec:threat} are injected rather than measured on real optics. We run no
lighting, viewpoint or material study, and the training-time patch augmentation is not
identical to the attack's own. On a real system the stride $h$ is bounded by inference
latency and need not stay constant. Ours is fixed and known: the sweep covers
$h\in\{1,\dots,4\}$, and hence $M$ from $15$ down to $3$, but never a stride that varies
within an episode.
We test two policy families, Diffusion Policy and ACT. The mechanism needs only that a
timestep be covered by several predictions that are then combined, and it does not apply
where a query emits a
single action or discards the overlap instead of aggregating it.
We test one defence instantiation per family.

\textbf{The threat model.} Sparse corruption is an assumption about the redundancy, argued
from the intermittent \emph{pattern} such faults produce and from intermittent access. The
fault that separates the aggregators, blank frames, we evaluate on one task and one
stride.

\section{Conclusion}
\label{sec:conclusion}

Action chunking already supplies the redundancy that robust aggregation needs, and the
deployed temporal-ensembling rule has breakdown point $0$. The coordinate-wise median is a one-line
replacement with a \emph{combinatorial} guarantee, verified at the level of the statistic
and stable across a tenfold increase in attack optimisation effort, where a metric-oriented
defence loses most of its effect. We also delimit where it cannot help. Corruption that
shifts every covering candidate equally is invisible to the whole $\Delta$-family, and no
equivariant aggregator can remove it.
An oracle attacker confined to that bracket does not make the task fail, though it does
more damage when it pushes the aggregate the same way at consecutive timesteps.

\section*{ACKNOWLEDGMENT}
A large language model was used in producing this work: for code, for the design of the
figures and tables, and for drafting and polishing the text. The author reviewed and
edited all of it and is responsible for the content.

\bibliographystyle{IEEEtran}
\bibliography{refs}

\end{document}